\documentclass{article} 
\usepackage{nips15submit_e,times}
\usepackage{hyperref}
\usepackage{url}
\usepackage{amsfonts, amsmath, amssymb, amsthm, constants,comment}
\usepackage{algorithm}
\usepackage{algorithmic}

\usepackage{tikz}
\usepackage{tikz}
\usetikzlibrary{arrows,calc}
\usetikzlibrary{mindmap,trees,backgrounds}
\usetikzlibrary{positioning}
\usepackage{verbatim}
\usepackage{subfigure}
\usepackage{bbm}
\newtheorem{theorem}{Theorem}
\newtheorem{definition}{Definition}
\newtheorem{lemma}{Lemma}

\newtheorem{remark}{Remark}

\providecommand{\nor}[1]{\left\lVert {#1} \right\rVert}
\providecommand{\abs}[1]{\lvert{#1}\rvert}

\providecommand{\scalT}[2]{\left\langle{#1},{#2}\right\rangle}

\newcommand{\EE}{\mathcal E}
\def\argmax{\operatornamewithlimits{arg\,max}}
\def\argmin{\operatornamewithlimits{arg\,min}}

\newcommand{\XX}{\mathcal X}

\newcommand{\YY}{\mathcal Y}

\newcommand{\eps}{\varepsilon}

\title{ Random Maxout Features}

\author{
Youssef Mroueh ~ Steven Rennie~ Vaibhava Goel\\
IBM T.J Watson Research Center\\
\texttt{$\{$mroueh,sjrennie,vgoel$\}$@us.ibm.com} }

\nipsfinalcopy 

\begin{document}

\maketitle

\begin{abstract} 
In this paper, we propose and  study random maxout features, which are constructed by first projecting the input data onto sets of randomly generated vectors with Gaussian elements, and then outputing the maximum projection value for each set.
We show that the resulting random feature map, when used in conjunction with linear models, allows for the locally linear estimation of the  function of interest in classification tasks, and for the locally linear embedding of points when used for dimensionality reduction or 
data visualization.
We derive generalization bounds for learning that assess the error in approximating locally linear functions by linear functions in the maxout feature space, and empirically evaluate the efficacy of the approach on the MNIST and TIMIT classification tasks.

\end{abstract} 
\section{Introduction}
Kernel based learning algorithms are ubiquitous in both supervised and unsupervised learning. For example, a universal kernel support vector machine approximates, to an arbitrary precision, any non-linear decision boundary function given enough training points \cite{vapnik98}. On the other hand, methods like Kernel Principal Component Analysis \cite{Scholkopf99kernelprincipal} (Kernel PCA) capture non-linear relationships between variables of interest, and  are used in non-linear dimensionality reduction. However, non-linear kernel methods suffer from high computational complexity (often cubic in the sample size), and are difficult to parallelize---training and testing on a even modestly sized dataset such as the TIMIT speech corpus (2M training samples) can be very challenging.
Linear methods, on the other hand  (linear support vector machines, logistic regression, ridge regression, Principal Component Analysis, etc. ),  suffer from low capacity and representation power, but have computational complexity linear in the sample size, and so can be more readily scaled to large data corpora.
Scalability and non-linear representation power are two desiderata of any learning algorithm. Deep Neural Networks owe their success to this property, as they allow  rich, non-linear modeling and they scale linearly in the sample size when trained with variants of stochastic gradient descent \cite{lecun}.

\noindent \textbf{Kernel Approximation with Random Features.} An elegant approach to overcoming the computational load of kernel methods, pioneered by \cite{Recht07}, consists of generating explicit, randomized feature maps $\Phi:\mathcal{X}\subset \mathbb{R}^d \to\mathbb{R}^m$, where $m$ is typically larger then the dimension of the input space $n$, to approximate the kernel $K$:
\begin{equation}
\text{For } (x,z) \in \mathcal{X}, \quad K(x,z)\approx \scalT{\Phi(x)}{\Phi(z)}.
\end{equation}
When used in conjunction with linear methods, such randomized features reveal the non-linear structure in the data, and we gain scalability, as  linear methods scale linearly with training sample size. 
Random Fourier features, introduced in  \cite{Recht07}, approximate shift invariant kernels. For example, the Gaussian kernel,
$K(x,z)=\exp\left(-\frac{\nor{x-z}^2}{2\sigma^2}\right),$
 can be approximated using the following feature map: 
$$\Phi(x)=(\cos(w_1^{\top}x+b_1) \dots \cos(w_m^{\top}x+b_m)),$$
where $w_i\sim \mathcal{N}(0,\frac{1}{\sigma^2}I_d)$, are independent gaussian vectors, and $b_i$ are independently and uniformly drawn  from $[0,2\pi]$.
Recently \cite{timitRF} showed that a highly oversampled random Fourier features map $(m=400K)$, and a large scale linear least squares classifier,  approaches the performance of dense deep neural networks trained on the TIMIT speech corpus.

\noindent \textbf{Learning with Random Features.} More formally in a classical supervised learning setting, let $\mathcal{X} \subset \mathbb{R}^d$, be the input space and $\mathcal{Y}=\{-1,1\}$ be the label space in a binary classification setting. We are given a training set $S=\{(x_i,y_i)\in \mathcal{X}\times \mathcal{Y}, i= 1\dots N\}$. 
For kernel methods, the goal is to find a non-linear function $f$ mapping $\mathcal{X}$ to $\mathcal{Y}$, given a  certain measure of discrepancy or a loss function $V$. The function $f$ is restricted to belong to a hypothesis class of functions $\mathcal{H}_{K}$, the so called  reproducing kernel Hilbert space (RKHS). Empirical risk minimization in that setup  leads to a rich class of non-linear algorithms, via regularization in RKHS \cite{wahba90},
 \begin{equation}\label{eq:LSK}
\min_{f \in \mathcal{H}_{K}} \frac{1}{N}\sum_{i=1}^N V(y_i,f(x_i)) +\lambda ||f||^2_{\mathcal{H}_K},
\end{equation}
 where $\lambda>0$ is the regularization parameter and $\nor{f}_{\mathcal{H}_{K}}$ is the norm in the RKHS. The optimum $f^*$ of the problem in  \eqref{eq:LSK} has the following form $f^*(x)=\sum_{i=1}^N \beta^*_i K(x,x_i), \beta^* \in \mathbb{R}^{N}.$ Solving for $\beta^*$ may have a computational complexity of $O(N^3)$ (Regularized Least squares) or $O(N^2)$ (Support Vector Machines). 
\noindent Using an explicit feature map $\Phi$ that approximates such a kernel in conjunction with a linear model, it is therefore sufficient to estimate a scalable regularized linear model with a computational complexity linear in the number of training examples,
\begin{equation}\label{eq:LSPHI}
\min_{\alpha \in \mathbb{R}^m}\frac{1}{N}\sum_{i=1}^N V(y_i, \scalT{\alpha}{\Phi(x_i)}) +\lambda \nor{\alpha}^2,
\end{equation}
where $\alpha^*$ is the optimal solution. For sufficiently large $m$ we have: $f^{*}(x)\approx \scalT{\alpha^*}{\Phi(x)}$, and  $\alpha^*$ can be find in $O(Nm)$ time using stochastic gradient descent. Similar ideas extend to the unsupervised case. Recently \cite{RCA} introduced the randomized non-linear component analysis, where it is shown that  Kernel PCA can be approximated by using random Fourier features followed by a linear PCA.
  
  \noindent \textbf{Contributions.} In this paper, inspired by the recently introduced maxout network \cite{Maxout}, we introduce a simple but effective non-linear random feature map, called random maxout features, that approximates functions of interest with piecewise-linear functions. Locally linear boundaries and components  are interesting as they carry locally the linear structure in the data, and have the advantage of being interpretable in the original feature space of the data, but how should such a kernel method be formulated? In principle, such a mapping could be achieved via any locally linear kernel of the form $K(x,z)=\scalT{x}{z}\kappa_{\sigma}(x,z),$ where $\kappa_{\sigma}$ is a localizing kernel, such as, for example a Gaussian kernel with bandwidth $\sigma$. However, how to efficiently realize such a conditionally linear kernel is not clear; for example, achieving this via random Fourier features would involve taking the Kronecker product of a linear feature map and the random features. The main contribution of this paper is to introduce and analyze the random maxout feature map, which has the advantage that it can be learned in $O(Nm)$ time ($N$ training points, $m$ random features), and utilized at test time in $O(dm)$ time (assuming $d$ input features), while avoiding the taking Kronecker products.
When used in conjunction with linear methods, random maxout realize a scalable, local linear function estimator for large-scale classification and regression. In the unsupervised setting, similarly to \cite{RCA}, random maxout features followed by a PCA allow for a locally linear embedding of the data, that can be used as a non-linear dimensionality reduction, and for data visualization. 
The paper is organized as follows: In Section \ref{sec:RMaxout} we introduce our random maxout feature map, and show that its expected kernel is indeed locally linear. In section \ref{sec:Learning} we present generalization bounds for the learning of linear functions in the random maxout feature space. In section \ref{sec:prevwork}, we discuss how random maxout features relate to previous work. Finally, in section \ref{sec:app}, we demonstrate the approach as a local linear estimator in a classification setting  on MNIST and TIMIT speech corpora.  

\section{Random Maxout Features}\label{sec:RMaxout}
Random maxout features have the same structure as deep maxout networks in terms of maxout units .
The following definition gives a precise description of the maxout random feature map:
\begin{definition}[Random Maxout Features]\label{def:RMOF}
Let $w^{\ell}_{j}$, $\ell=1\dots m$, and $j=1\dots q$, be independent random gaussian vectors i.e $w^{\ell}_j\sim \mathcal{N}(0,I_d)$. Note $W^{\ell}=(w^{\ell}_1\dots w^{\ell}_q)$\\
For $x \in \mathbb{R}^{d}$, we define a maxout random unit $h_{\ell}(x)$ as follows: 
\begin{equation*}
h_{\ell}(x)=\phi(x,W^{\ell})=\max_{j=1\dots q}\scalT{w^{\ell}_j}{x}, \quad \ell =1\dots m.
\end{equation*}
A maxout random feature map $\Phi$ is therefore defined as follows:
\begin{equation*}
\Phi(x)=\frac{1}{\sqrt{m }}(h_{1}(x),\dots,h_{m}(x) ).
\end{equation*}
\end{definition}
\noindent In order to study this map, we shall consider, for $2$ points $x,z \in \mathbb{R}^d$, the dot product: 
$\scalT{\Phi(x)}{\Phi(z)}=\frac{1}{m}\sum_{\ell=1}^m h_{\ell}(x)h_{\ell}(z).$
Consider first the expectation of $\scalT{\Phi(x)}{\Phi(z)}$:
$\mathbb{E}(\scalT{\Phi(x)}{\Phi(z)})=\frac{1}{m}\sum_{\ell=1}^m \mathbb{E}(h_{\ell}(x)h_{\ell}(z))= \mathbb{E}(h_1(x)h_1(z)),$
where the last equality follows from the independence of the units. It It is therefore sufficient to study the expectation of the dot product of one unit:
\begin{equation*}
K(x,z)=\mathbb{E}(h(x)h(z)),
\end{equation*}
where $h(x)=\max_{j=1\dots q}\scalT{w_j}{x}, w_j \sim \mathcal{N}(0,I_d), j=1\dots q,$ iids.

\begin{theorem}[Maxout Expected Kernel]\label{pro:Exp}
Let $x,z \in \mathbb{R}^d$.
The expected kernel of maxout random units is given by the following expression:
\begin{equation*}
K(x,z)= \mathbb{E}(h(x)h(z))= \sigma^2(q) \scalT{x}{z} \kappa_q(x,z),
\end{equation*}
where $\sigma^2(q)=\mathbb{E}(\max_{j=1\dots q} g_j)^2,~ g_j\sim \mathcal{N}(0,1)$ iid , and $\kappa_{q}(x,z)$ is a non-linear kernel given by:
\begin{equation*}
\kappa_q(x,z)=\mathbb{P}\left\{\argmax_{j=1\dots q}\scalT{w_j}{x}=\argmax_{j=1\dots q}\scalT{w_j}{z}\right\}=\sum_{i=0}^{\infty}a_i(q)\left(\frac{\scalT{x}{z}}{\nor{x}\nor{z}}\right)^i,
\end{equation*}
where the first $3$ coefficients are $a_0(q)=\frac{1}{q},a_1(q)=\frac{h_1^2(q)}{q-1}$, $a_2(q)=\frac{qh_2^2(q)}{(q-1)(q-2)}$, where
$h_{i}(q)=\mathbb{E}\phi_i(\max_{k=1\dots q}g_k))$, where $g_{j},j=1\dots q$ are iid standard centered gaussian, and $\phi_{i}$, the normalized Hermite polynomials. $a_i(q)$ are non negative and $\sum_{i\geq 0}a_i(q)=1$.
\end{theorem}
\begin{proof}[Proof of Theorem \ref{pro:Exp}]
The proof is given in Appendix A in the supplementary material.
\end{proof}
\subsection{Discussion of the Derived Maxout Kernel}\label{sec:eqKernel}

\noindent The expected kernel of a maxout unit is therefore a locally weighted linear kernel, and hence it allows  a non-linear estimation of functions in a piecewise linear way :
$$K(x,z) =\sigma^2(q) \scalT{x}{z}\kappa_{q}(x,z),$$
where $\kappa_q$ is a non-linear kernel. Let $\rho=\scalT{x}{z}$. In this section we discuss the locality introduced by $\kappa_{q}(.,.)$.\\
\noindent It is important to note that $0\leq\kappa_{q}(x,z)\leq 1$, since $\kappa_{q}(x,z)=\mathbb{P}(D(x)=D(z))$, where $D(x)=\arg\max_{j=1\dots q}\scalT{w_j}{x}$.
For simplicity assume that $x,z \in \mathbb{S}^{d-1}$, where $\mathbb{S}^{d-1}$ is the unit sphere in $d$ dimensions. We start by giving values of $\kappa_{q}$ in three particular cases of interest:
\begin{enumerate}
\item When $x$ and $z$ coincide i.e $x=z$, and $\rho=1$, we have $~\kappa_{q}(x,z)=1$, as $\sum_{i\geq0}a_{i}(q)=1$.
\item When $x$ and $z$ are orthogonal i.e $\rho=0$, we have $\kappa_{q}(x,z)=a_{0}(q)=\frac{1}{q}$.
\item When $x$ and $z$ are \rm{diam}etrically opposed i.e $x=-z$, and $\rho=-1$, we have $~\kappa_{q}(x,z)=0$, as $\sum_{i\geq0}a_{2i}(q)=\sum_{i\geq0}a_{2i+1}(q)=\frac{1}{2}$  \cite{Frieze95} .
\end{enumerate}



In order to understand the locality introduced by the non-linear kernel $\kappa_q$, and the relation of the radius of the locality to the size of the pool $q$,
looking to the first order expansion of $\kappa_q$ gives us a hint on the effect of that kernel. In particular the quantity $h_{1}(q)$ is  just the expectation of the maximum of independent gaussians $h_{1}(q)\sim\sqrt{2\log(q)}$ \cite{galambos,Frieze95}.
\begin{eqnarray*}
\kappa_{q}(x,z)&=&a_0(q)+a_{1}(q)\rho+O(\rho^2)\\
&=& \frac{1}{q}\left(1+(1+\epsilon(q))2\log(q)\rho\right)+O(\rho^2),
\end{eqnarray*}
where $\epsilon(q)\to 0,$ for $q\to \infty$.\\
Note by $g$ the function, $g: [-1,1]\to [0,1]$ such that $\kappa_{q}(x,z)=g(\rho)$. For far apart points, when $\rho \to -1$, $g(\rho)\to 0$. $g$ has a linear behavior around $\rho=0$, with a slope equal to $\frac{2\log(q)}{q}$. Note that in this neighborhood as $q$ increases the linear regime vanishes, and $g(\rho) \to 0$. Hence as $q$ increases the probability of two points hashing to same index of  maximum becomes smaller; qualitatively  the radius of the locality of $\kappa_{q}$ shrinks as $q$ increases.  Finally for  near by points when $\rho\to 1$, $g(\rho)\to 1$.\\
Qualitatively the derived kernel $K(x,z)\approx 0$ for far apart points and $K(x,z)\approx \sigma^2(q)\scalT{x}{z}$ for points in the same neighborhood, where the radius of the locality, and  the notion of closeness is set by the choice of the size of the pool $q$. This radius is decreasing in $q$. Hence $K$ defines a locally linear kernel.\\
\noindent Now if we go back to problem \eqref{eq:LSK}, and solve for $f$ in the reproducing kernel hilbert space of the equivalent kernel $K$ (i.e for $\mathcal{H}=\mathcal{H}_{K}$), we have  :
\begin{eqnarray}
f^*(x)=\sum_{i=1}^N\beta^*_iK(x,x_i)=\sigma^2(q)\sum_{i=1}^N \beta^*_i\scalT{x}{x_i}\kappa_{q}(x,x_i)= \sigma^2(q)\scalT{\sum_{i=1}^N \beta^*_i \kappa_{q}(x,x_i)x_i}{x}.
\label{eq:LocalLinear}
\end{eqnarray} 
\noindent Hence we see that this derived kernel allows a locally linear estimation of the function of interest $f^*$, where the radius of the locality is set by the choice of the size of the pool $q$.\\
\noindent Now consider the maxout random feature map $\Phi$ introduced in Definition \ref{def:RMOF}. Recall that  we have: 
 \begin{eqnarray*}
 \mathbb{E}(\scalT{\Phi(x)}{\Phi(z)})= K(x,z)= \sigma^2(q)\scalT{x}{z}\kappa_{q}(x,z),
\end{eqnarray*}
the dot product $\scalT{\Phi(x)}{\Phi(z)}$ is therefore an estimator of $K(x,z)$ i.e for sufficiently large $m$, $\scalT{\Phi(x)}{\Phi(z)}\approx K(x,z)$ , hence we can use the feature map $\Phi$, and a simple linear model  as in equation \eqref{eq:LSPHI} , and use the optimal weight $\alpha^*$ to get an estimate  of the locally linear estimation $f^*$ produced by the derived kernel as in equation \eqref{eq:LocalLinear}, i.e. we have for sufficiently large $m$:
\begin{equation}
\scalT{\alpha^*}{\Phi(x)}\approx f^*(x).
\end{equation}
In the next section we analyze the errors incurred by such approximation and how it translates to the convergence of the risk to its optimal value in a dense subset of the RKHS induced by the locally linear kernel.
\begin{remark}[Locality Sensitive Hashing]
Let $C: \mathbb{S}^{d-1}\to \{1\dots q\}^m$, such that for $x\in \mathbb{S}^{d-1}$,  $C(x)=\left(\argmax_{j=1\dots q} \scalT{w^{1}_j}{x},\dots, \argmax_{j=1\dots q}\scalT{w^{m}_j}{x}\right),\\ w^{\ell}_j\sim \mathcal{N}(0,I_d),\ell=1\dots m, j=1\dots q.$
For $x,z \in \mathbb{S}^{d-1}$ we have : $\mathbb{E}\left(\frac{1}{m}\sum_{i=1}^m \mathbbm{1}_{C_i(x)\neq C_i(z)}\right)=\mathbb{P}\left\{ \argmax_{j=1\dots q}\scalT{w_j}{x}\neq \argmax_{j=1\dots q}\scalT{w_j}{x}\right\}=1-\kappa_{q}(x,z).$
Hence we can approximate the local kernel $\kappa_{q}(x,z)$, by the non binary strings by mean of the hamming distance between the $q$-ary strings $C(x)$, and C(z). As $m$ becomes large we have:
$$d_{H}(C(x),C(z))=\frac{1}{m}\sum_{i=1}^m \mathbbm{1}_{C_i(x)\neq C_i(z)}\approx 1-\kappa_{q}(x,z),$$
Hence $C$ defines a locality sensitive hashing scheme in the sense of \cite{Indyk01} .
\end{remark}

\section{Learning with Random Maxout Features}\label{sec:Learning}
We show in this section that learning a linear model in the random maxout feature space, allows for a locally linear  estimation of functions in a supervised classification setting. 
The locally linear kernel $K(x,z) =\sigma^2(q) \scalT{x}{z} \kappa_{q}(x,z)$, defines a Reproducing Kernel Hilbert Space (RKHS). In the following we will see how linear functions in the random maxout feature space approximate a dense subset of this locally linear RKHS. We start by introducing some notation.
We assume that we are given a training set $S=\{(x_i,y_i), x_i \in \mathcal{M}=\mathcal{X}\cap \mathbb{S}^{d-1}, y_i \in \YY= \{-1,1\}, i=1\dots N\}$. Our goal is to learn a function $f :\mathcal{M}\to \mathbb{R}$ via risk minimization. Let $\rho_{y}(x)$ be the label posteriors and assume $\mathcal{M}$ is endowed with a measure $\rho_{\mathcal{M}}$, the expected and empirical risks induced by a $L$-Lipchitz loss function $V: \mathbb{R}\to [0,1]$ are the following: $$\mathcal{E}_{V}(f)=\int_{\mathcal{M}}\sum_{y\in \YY}V(yf(x))\rho_{y}(x)d\rho_{\mathcal{M}}(x),~ \hat{\EE}_{V}(f)=\frac{1}{N}\sum_{i=1}^N V(y_if(x_i)).$$
 The assumptions on the points belonging to  the unit sphere, and on the loss being bounded by one can be weakened  see Remark \ref{rem:Bound}.
We will use in the following a notion of intrinsic dimension for the set $\mathcal{M}$, namely the Assouad dimension given in the following definition:
\begin{definition}[\cite{assouad}]
 The Assouad dimension of $\mathcal{M}\subset \mathbb{R}^d$ , denoted by $d_{\mathcal{M}}$ , is the smallest integer $k$, such that, for any ball $B\subset \mathbb{R}^d$,the set $B\cap \mathcal{M} $ can be covered by $2^k$ balls of half the radius of $B$.
 \end{definition}
The Assouad dimension is used as a measure of the intrinsic dimension. For example, if  $\mathcal{M}$ is an $\ell_{p}$ ball in $\mathbb{R}^d$, then $d_{\mathcal{M}} = O(d)$. If $\mathcal{M}$ is a $r$-dimensional hyperplane in $\mathbb{R}^r$, then $d_{\mathcal{M}} = O(r)$, where $r<d$. Moreover, if $\mathcal{M}$ is a $r$-dimensional Riemannian submanifold of $\mathbb{R}^d$ with  suitably bounded curvature, then $d_{\mathcal{M}} =O(r)$.\\
Let $W^{\ell}=(w^{\ell}_{1},\dots w^{\ell}_{q})$, $\ell=1\dots m$, and $W=\left( w_1,\dots w_{q}\right)$, since $w_j,j=1\dots q$ are iid, the distribution of $W$ is given by $p(W) =p(w_1)\dots p(w_{q})$, where $p(w_j)$ is the distribution of a gaussian vector drawn form $\mathcal{N}(0,I_d)$. Similarly to the analysis in \cite{Rah_Rec:2008:allerton}, let $C>0$, we define the infinite dimensional functional space $\mathcal{F}$:
$$\mathcal{F}=\left\{f(x)=\int \alpha(W) \phi(x ,W) dW,~ \sup_{W} \frac{\left|\alpha(W)\right|}{p(W)}\leq C \right\},$$
it is easy to see that $\mathcal{F}$ is dense in $\mathcal{H}_{K}$ \cite{Rah_Rec:2008:allerton}.
We will approximate the set $\mathcal{F}$ with $\hat{\mathcal{F}}$ defined as follows:
$$\hat{\mathcal{F}}=\left\{f(x)=\sqrt{m}\scalT{\alpha}{\Phi(x)}=\sum_{\ell=1}^m \alpha_{\ell} \phi(x ,W^{\ell}) , \nor{\alpha}_{\infty}\leq \frac{C}{m} \right\}.$$
Note that in this definition of this function space we are regularizing the norm infinity of the weight vectors this can be replaced in practice, and theory \cite{Bach15} by a classical Tikhonov regularization or other forms of regularization.
\begin{theorem}[Learning with Random Maxout Features]\label{theo:Learning} Let $S=\{(x_i,y_i), x_i \in \mathcal{M}=\mathcal{X}\cap \mathbb{S}^{d-1}, y_i \in \{-1,1\}, i=1\dots N\}$, and $d_{\mathcal{M}}$ the assouad dimension of $\mathcal{M}$,and $\rm{diam}(\mathcal{M})$ be its diameter. Let $\hat{f}_{N}=\argmin_{f\in \mathcal{\hat{F}}} \hat{\mathcal{E}}_{V}(f)$. Fix $\delta>0$, $\eps\in (0,1)$, for $m \geq \frac{C'}{\eps^2} \left(d_{\mathcal{M}}\log \left(\frac{{\rm{diam}}(\mathcal{M})\sqrt{d}}{\delta}\right)+\log(q+1)\right)$, where $C'$ is a numerical constant, we have: 
\begin{align*}
\mathcal{E}_{V}(\hat{f}_{N})- \min_{f\in \mathcal{F}}\mathcal{E}_{V}(\hat{f}) &\leq  4LC \sqrt{\frac{\sigma^2(q)}{N}}+ \frac{2|V(0)|}{\sqrt{N}}+ 2\sqrt{\frac{2\log(1/\delta)}{N}}
+ LC \eps \left(1+\sqrt{2\log\left(\frac{1}{\delta}\right)}\right),
\end{align*}
with probability at least $1-3\delta -2 e^{-cd/4}$, on the choices of the training examples and the random projections.  Where $C$ is the regularization parameter in the definition of $\mathcal{F}$ and $\hat{\mathcal{F}}$, $L$ is the lipchitz constant of the loss function $V:\mathbb{R}\to [0,1]$, and $\sigma^2(q)=\mathbb{E}\left(\max_{j=1\dots q} g_j\right)^2, g_j \sim \mathcal{N}(0,1)$ iids.
\end{theorem}
The proof of Theorem \ref{theo:Learning} is given in the supplementary material in Appendix B, the main technical difficulty consists in bounding the $\sup_{x\in \mathcal{M}}\left| \phi(x,W)\right|,$ and relating this quantity to the intrinsic dimension $d_{\mathcal{M}}$.
Theorem \ref{theo:Learning}, shows that for $q>1$, learning a linear model in the maxout feature space defined by the map $\Phi$ has a low expected risk and more importantly this risk is not far from the one achieved by a nonlinear infinite dimensional function class $\mathcal{F}$. Locally linear functions can be hence estimated to an arbitrary precision using linear models in the maxout feature space, the errors decompose naturally to an estimation or a  statistical error with the usual rate of $O(\frac{1}{\sqrt{N}})$, and an approximation error  of functions in the infinite dimensional space $\mathcal{F}$, by functions in $\hat{\mathcal{F}}$.
For a fixed $q$, in order to achieve an approximation error $\eps$, the bounds suggests $\eps=\frac{1}{\sqrt{N}}$. One needs to set the dimensionality $m$  of the feature map $\Phi$  to $O \left(N\left(d_{\mathcal{M}} \log(d)+\log(q+1)\right)\right)$, where $d_{\mathcal{M}}$ is a measure of the intrinsic dimension of the space where inputs live $d_{\mathcal{M}}\leq d$. For instance if our data lived on a  $r$-dimensional Riemannian submanifold of $\mathbb{R}^d$$(r<<d)$, the function space $\hat{\mathcal{F}}$ for $m=O\left(N\left(r\log(d)+\log(q)\right)\right)$ achieves an approximation error $\eps=\frac{1}{\sqrt{N}}$ of  a dense subset of the function space defined by the local linear kernel, with a radius of locality set by the choice of the parameter $q$. As $q$ increases this radius shrinks and the dimension of the feature map increases to ensure more locality but with a logarithmic dependency on $q$. The use of the intrinsic dimension of the inputs space $\mathcal{M}$ -that we borrow from the compressive sensing community- in the approximation error is appealing as most of previous bounds in random features analysis relates the number of projections only to the  training size $N$, and spectral properties of the kernel matrix \cite{Bach15},\cite{Rah_Rec:2008:allerton}. Using spectral propreties of the kernel, results in \cite{Bach15} suggest that for large $N$ that the number of the features if of the order $O(N\log(N))$ . While the spectral properties of the kernel  carry some geometric information about the points distribution it misses some important geometric structure in the points set $\mathcal{M}$, since it captures some intrinsic dimension of the data that can be expressed only in term of the sample size $N$, while the Assouad dimension has a richer description of the structure in the data, such as sparsity for instance. If $\mathcal{X}$ was the set of $s-$ sparse signal the Assouad dimension $d_{\mathcal{M}}=O(s\log(d))$\cite{assouad}, and we need $O\left(\frac{s\log^2(d)+\log(q)}{\eps^2}\right)$ maxout random features to have an approximation error of $\eps$. It would be interesting to incorporate in the bound both the spectral properties of the kernel and the intrinsic dimension to get the good of the two worlds, we leave this for a future work. 
For $q=1$, Maxout random features reduces to classical random projections, that approximate the linear kernel, learning classifiers  from randomly  projected data has been throughly studied see  \cite{durrant2013sharp}, and references there in, Theorem \ref{theo:Learning} is not as sharp as results presented in \cite{durrant2013sharp}, since the proof was not specialized to the linear projection case.
\begin{remark} \label{rem:Bound}1-We can relax the sphere constraint on the input set to a bounded data constraint, i.e $\sup_{x\in \XX} \nor{x} \leq R$, and assume a bounded loss $|V(z)|\leq B$, a minor change in the proof  shows that the right hand side of the inequality in Theorem \ref{theo:Learning} is multiplied by $RB$. \\
2-Note that for $\eps=\frac{1}{\sqrt{N}}$, we have $m=O(Nd_{\mathcal{M}}\log(d))$, for large $N$ assume $\mathcal{M}$ was finite and the cardinality $|\mathcal{M}|=N^{\alpha}$ for small $\alpha$, we have $d_{\mathcal{M}}=O(\log(|\mathcal{M}|))= O(\log(N))$ and $m=O(N\log(N)\log(d))$, which matches up to a log term results in \cite{Bach15}.  
\end{remark}
\section{Related Work}\label{sec:prevwork}
\textbf{Approximating Kernels, Random Non Linear Embeddings.} The so called Johnson-Lindenstrauss Lemma \cite{JL} states that a  linear random feature map preserves $\ell_2$ distances in a $N-$ point subset of a Euclidian space when embedded in 
$O(\eps^{-2}\log(N))$ dimension with a distortion of $1+\eps$. The requirement of preserving all  pairwise distances is not needed in many applications; we need to preserve distances only in a local neighborhood of the points of interest. This observation is at the core of locality sensitive hashing \cite{Indyk01} and has been discussed in \cite{conf/soda/BartalRS11}. One needs a non-linear random feature map in order to achieve a local embedding. Random Fourier features \cite{Recht07} approximating the Gaussian kernel achieve such a goal. Random maxout  features also achieve such a goal by performing a locally linear embedding of the points. \\
\textbf{Scaling up Kernel Methods.} As discussed earlier random features is a popular approach in  approximating the kernel matrix and scaling up kernel methods pioneered by \cite{Recht07}, the generalization ability of such approach is of the order of $\tilde{O}(\frac{1}{\sqrt{N}}+\frac{1}{\sqrt{m}})$, which suggests that $m$ needs to be $\tilde{O}(N)$. An elegant doubly stochastic gradient approach introduced recently in \cite{DaiXHLRBS14}, uses random features to approximate the function space rather then the kernel matrix, in a memory efficient way that achieves this $O(N)$ bound for the number of features, Maxout random features can be also used within this framework.
Other approaches for scaling up kernel methods fall under the category  of  low rank Approximation of the kernel matrix, such as sparse greedy matrix approximation \cite{SmolaS00}, Nystrom approximations \cite {Williams01usingthe} and low rank Cholesky decomposition \cite{Fine01efficientsvm}. \\
\textbf{Locally Linear Estimation.} As discussed earlier, Random Maxout Networks allow us to do local linear estimation of functions in  supervised and unsupervised learning tasks, among other approaches Deep Maxout Networks \cite{Maxout}, Locally linear Embedding \cite{Roweis2000} and convex piecewise linear fitting \cite{Boyd}, share similar structure with Random maxout features.

\section{Numerical Experiments }\label{sec:app}
\subsection{Simulated Data Illustration}
In this section we consider $100$  points generated at random form the unit circle in two dimensions. We embed those points through the Maxout feature map $\Phi$, for $m=1000$, and  $q=2^3$ and $q=2^5$ respectively. We plot in Figure \ref{fig:locality}, for pairs of points $x$ and $z$, the pairwise distances in the embedded space $||\Phi(x)-\Phi(z)||$ versus the pairwise distances in the original $||x-z||$ (we show here only a subset of those pairwise distances). We see that in both cases for small distances we have a linear regime, followed by a saturation regime for high range distances. The saturation arises earlier for $q=2^5$ when compared to the one of $q=2^3$. This confirms the discussion in Section \ref{sec:eqKernel}, on the effect of the Maxout feature map as a locally linear kernel, with the radius of the locality shrinking as $q$ increases.
\begin{figure}[ht]
\centering
\includegraphics[width=0.4\linewidth]{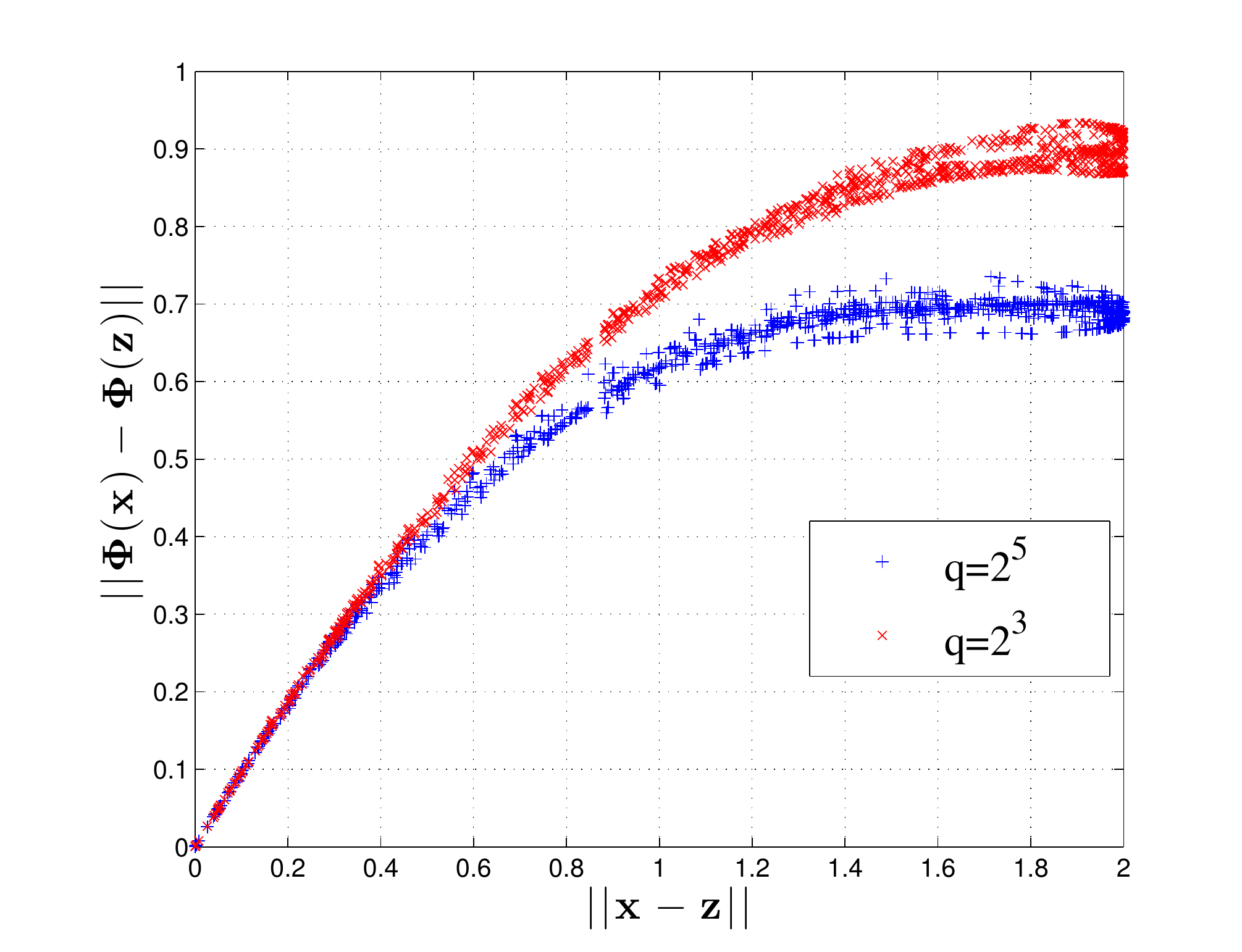}
\caption{The locality induced by the size of the pool $q$. As $q$ increase the locality radius shrinks.}
\label{fig:locality}
\end{figure}
To further illustrate this local linear embedding, we consider a dataset of $33$ images of faces of the  same person of size $112\times 92$ at different angles that we normalize to be unit norm (See Figure \ref{fig:exkernel}) .
 The faces are ordered by their angles, the ordering provided in this dataset is noisy. We extract the maxout features on this dataset for $m=10000,q=12$, and perform principal component analysis on the data in the maxout feature space and project it down to two dimensions on the two largest principal components. We show in Figure \ref{fig:exkernel}, the embedding of this dataset in two dimensions through Maxout followed by a linear PCA   . Each point in this scatter corresponds to a face, the numbering refers to the corresponding order in the given angle labeling. We see that Maxout local linear embedding self organizes the data with respect to the angle of variation, and corrects the noisy labeling. 
\begin{figure}[tb]
\begin{subfigure}
\centering
\includegraphics[width=0.5\linewidth]{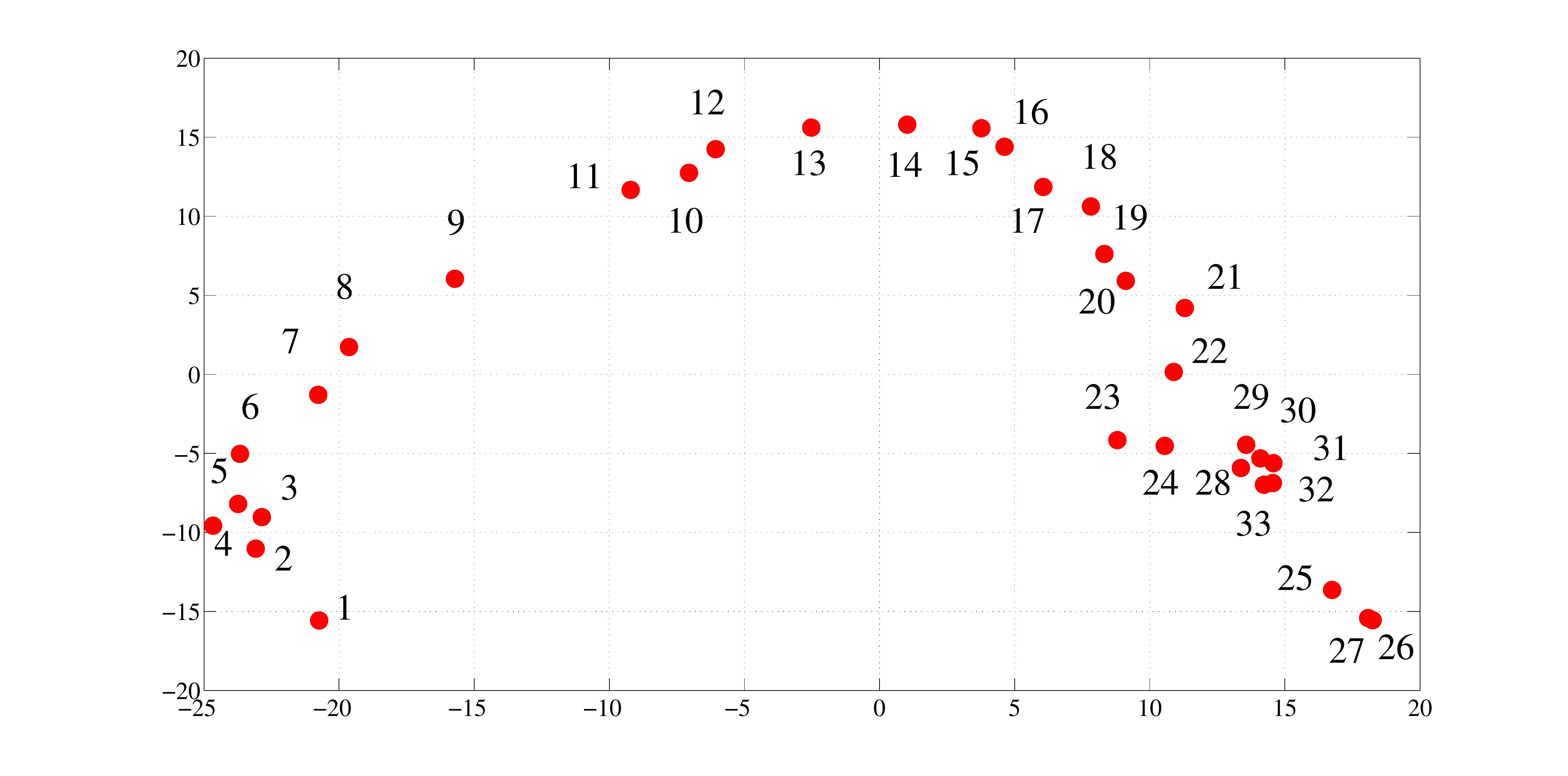}
\end{subfigure}
\begin{subfigure}
\centering
\includegraphics[width=0.5\linewidth]{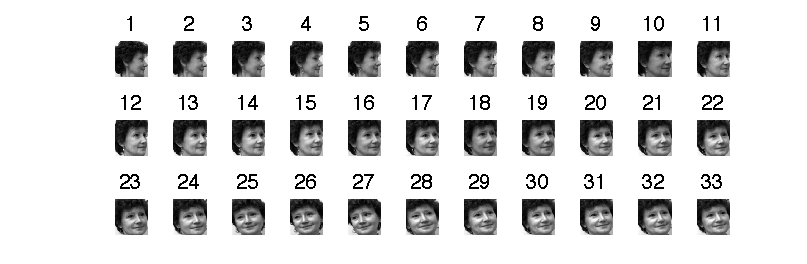}
\end{subfigure}
\caption{Maxout Locally Linear Embedding  in $2D$ (Maxout - LLE) of $33$ faces of the same person at different angles.}
\label{fig:exkernel}
\end{figure}

\subsection{Supervised Learning Applications:  Classification }\label{sec:TIMIT}
In order to perform classification as discussed earlier we lift the data through the maxout feature map $\Phi$, and solve a linear classification problem in the lifted space as described in Equation  \eqref{eq:LSPHI}.

\subsubsection{Digit Classification}
We extract the Maxout random features on the MNIST dataset \cite{lecun}, consisting of $60000$ training examples ($d=784$) and $10000$ test examples among $T=10$ digits. Let $Z=\Phi(X) \in \mathbb{R}^{N\times M}$be the embedded data and $Y\in \mathbb{R}^{N\times T}$ be the class labels using a $\pm1$ encoding. We solve the multi-class problem using a simple regularized least squares, $f(x)=\scalT{\alpha}{\Phi(x)}$, where $\alpha=(Z^{\top}Z+\lambda I)^{-1}Z^{\top}Y$, where $\lambda$ is the regularization parameter chosen on a hold out validation set \cite{gurls}.
\begin{table}[H]
\begin{center}
  \small
    \begin{tabular}{cccccc}
    \hline 
   & $m=100$ &$m=500$ & $m=1000$ & $m=5000$ &$m=10000$ \\ \hline
  $q=1$&$17.68\pm0.28$&$14.76\pm0.23$&$14.64\pm 0.08$&$15.02\pm 0.64$&$15.12\pm0.58$\\
  $q=2$&$16.95\pm0.5321$&$7.98\pm0.22$&$5.62\pm 0.18$&$2.86\pm 0.11$&$2.35\pm0.04$\\
 $q=4$& $18.24\pm0.39$&$7.70\pm0.24$&$5.50\pm0.16$&$2.78\pm0.07$&$\bf{2.23\pm0.07}$\\
 $q=8$&$18.9\pm0.28$&$7.98\pm0.47$&$5.59\pm0.19$&$2.81\pm0.10$&$2.32\pm0.09$\\
 $q=16$&$20.57\pm0.76$&$8.19\pm 0.21$&$5.66\pm0.15$&$2.87\pm0.12$&$2.58\pm0.04$
\end{tabular}
\end{center}
    \caption{Random Maxout on MNIST: Error rates in $\%$. }
    \label{tab:mnist}
\end{table}

%
In table \ref{tab:mnist} we report test errors and standard deviations for various values of $m$ and $q$ in the maxout feature map averaged on $5$ different choices of the random weights in the map. We see that for $q=1$, where the feature map does not introduce any non linearity, the performance of the map for any value of $m$, matches the error rate of a linear classifier that is  
$15 \%$. For $q\neq 1$, we start to see the non linearity introduced by the map as a local linear estimator, for a fixed $q$ the error rate decreases  as $m$ gets large. In this experiment the best error rate is achieved for  $m=10000$ and $q=4$, suggesting that $q=4$ sets the optimal radius of locality for classification. As a baseline an optimal $k-$ nearest neighbor achieves an error rate of $3.09~\%$.


\subsubsection{Phone Classification on TIMIT}
We further evaluated random maxout features on the TIMIT speech phone classification task. Evaluations are reported on the core test set of TIMIT. We utilized essentially the same experimental setup as in \cite{timitRF}: 147 context independent states were used as classification targets; at test time each utterance was decoded using the Viterbi algorithm, and then mapped, as is standard, down to 39 phones for scoring.  As in \cite{timitRF}, $2$ million frames of training data---fMLLR features of dimension $40$ each \cite{MohamedSDRHP11}. spliced with $\pm 5$ frames of context ($d=11\times40=440$)---were utilized. These features were then lifted hrough the random maxout map $\Phi$, and a multinomial logistic regression was subsequently trained using SGD to minimize cross entropy loss. Table \ref{tab:m_vs_q} reports the mean and standard deviation of the performance of random maxout units as a function of number of maxout features, $m$, and number of projections/feature, $q$. Interestingly, even smaller feature maps far outperform using the raw features, and the performance varies very little with initialization seed (5 seeds/result).

\begin{table}[H]
\begin{center}
  \small
    \begin{tabular}{ccccc}
    \hline \hline
\multicolumn{1}{c}{m} & \multicolumn{4}{c}{q} \\
 & 2& 4& 8& 16\\
1250 & 24.7$\pm$ 0.2 &
24.4$\pm$ 0.2 &
25.6$\pm$ 0.2 &
26.0$\pm$ 0.2 \\
2500 & 24.0$\pm$ 0.2 &
23.3$\pm$ 0.3 &
24.7$\pm$ 0.3 &
25.3$\pm$ 0.3 \\
5000 & 23.5$\pm$ 0.1 &
22.9$\pm$ 0.2 &
24.7$\pm$ 0.2 &
24.7$\pm$ 0.4 \\
10000 & 23.2$\pm$ 0.1 &
22.5$\pm$ 0.2 &
24.5$\pm$ 0.2 &
24.7$\pm$ 0.4 \\
20000 & 23.1$\pm$ 0.1 &
 22.3 $\pm$ 0.2 & 
24.3$\pm$ 0.2 &
24.5$\pm$ 0.2 \\

\hline
\hline
\end{tabular}
\end{center}
    \caption{Phone error rate (PER, \%) as a function of number of maxout features,m, and number of linear projections per maxout feature,q, on the TIMIT speech phone classification task. Multinomial logistic regression on the input features yields a PER of 33.1 $\pm$ 0.1\%.}
\label{tab:m_vs_q}
\end{table}

Table \ref{tab:large_scale} summarizes preliminary investigations into scaling up the size of the feature map, where to increase the number of features, projections are shared across random maxout units. Random maxout features appear to perform similarly to random Fourier features on the task. 

\begin{table}[H]
\begin{center}
  \small
    \begin{tabular}{ccccc}
    \hline \hline 
network & \# features (m) & \#projections & \#proj./feature (q) & phone error rate (PER) \\ \hline
Random Maxout & 15K & 15K &q=4      &  23.1   \\ 
Random Maxout & 60K & 15K&q=4   & 22.7    \\ 
Random Maxout & 60K& 60K&q=4      & 22.8   \\ 
Random Maxout & 400K& 15K&q=4 & 22.4     \\ 
Random Maxout & 300K & 300K&q=4    & 22.1 \\ \hline 
Random Fourier  & 400K & 400K&-         & 21.3 \cite{timitRF} \\
ReLU DNN & 4K, & 16K (4Kx 4 layers) & -             &    22.7 \cite{DahlSH13} \\
ReLU DNN w/ dropout & 4K, & 16K (4Kx 4 layers) & - &19.7 \cite{DahlSH13} \\
\hline
\hline
\end{tabular}
\end{center}
    \caption{Phone error rates (PER,\%) on TIMIT. The total number of projections used to produce each feature map are as indicated (here random maxout features draw from a shared pool of projections).}
\label{tab:large_scale}    
\end{table}

%
   
In this paper we presented random maxout feature map as an effective and scalable local linear estimator, and derived risk bounds for learning in this feature space that assesses 
both statistical and approximation errors, in a classification setting. We believe that maxout features, thanks to their conditionally linear structure, can gain further in scalability, and speed, by leveraging the fast  Johnson Lindenstrauss transform, and the doubly stochastic optimization framework of \cite{DaiXHLRBS14}.

\begin{appendix}
\section{Proof of Theorem $1$}

\begin{proof}[Proof of Theorem  $1$]
Assume without loss of generality that $||x||=||z||=1$.
Let $D(x)=\arg\max_{j=1\dots q}\scalT{w_j}{x}$, and $D(z)=\arg\max_{j=1\dots q}\scalT{w_j}{z}$, ties are broken arbitrarily.
By total probability we have:
\begin{align*}
K(x,z)&=\mathbb{E}(h(x)h(z))\\
&= \mathbb{E}\left \{h(x)h(z) | D(x)=D(z)\right\}\mathbb{P}(D(x)=D(z))\\
&+ \mathbb{E}\left\{h(x)h(z)|D(x)\neq D(z)\right\}\mathbb{P}(D(x)\neq D(z))\\
&= \mathbb{E}(\scalT{w_{D(x)}}{x} \scalT{w_{D(x)}}{z}|D(x)=D(z))\mathbb{P}(D(x)=D(z))\\
&+ \mathbb{E}\left\{\scalT{w_{D(x)}}{x} \scalT{w_{D(z)}}{z}|D(x)\neq D(z)\right\}\mathbb{P}(D(x)\neq D(z)).\\
\end{align*}
It is easy to see that the second term in this sum is zero since the gaussians are independent and zero centered : $\mathbb{E}\left\{\scalT{w_{D(x)}}{x} \scalT{w_{D(z)}}{z}|D(x)\neq D(z)\right\}=0$.
We are left with the first term of this sum: 
\begin{align*}
K(x,z)&=\mathbb{E}(h(x)h(z))\\
&=\mathbb{E}(\scalT{w_{D(x)}}{x} \scalT{w_{D(x)}}{z}|D(x)=D(z))\mathbb{P}(D(x)=D(z))\\
&= q \mathbb{E}(\scalT{w_{1}}{x} \scalT{w_{1}}{z}| D(x)=D(z)=1)\mathbb{P}(D(x)=D(z)=1)\\
&=  \mathbb{E}(\scalT{w_{1}}{x} \scalT{w_{1}}{z}| D(x)=D(z)=1) \mathbb{P}(D(x)=D(z)).
\end{align*}
By rotation invariance of gaussians we have:
$$\scalT{w_1}{x}=g  \text{ and } \scalT{w_1}{z} =\scalT{x}{z} g +\sqrt{1-|\scalT{x}{z}|^2}h,$$ where $g$ and $h$  are independent random gaussian variables $g,h \sim \mathcal{N}(0,1)$.\\
Let $E$ be the following event : 
$$E=\{g \text{ is the maximum of $q$ independent gaussians} \}$$
Hence we have:
\begin{align*}
&\mathbb{E}(\scalT{w_{1}}{x} \scalT{w_{1}}{z}| D(x)=D(z)=1)\\
 &=\mathbb{E}\left(g (\scalT{x}{z}) g + \sqrt{1-|\scalT{x}{z}|^2}h) | E \right)\\
&= \scalT{x}{z} \mathbb{E}\left( \left[\max_{j=1\dots q} g_j \right]^2\right).
\end{align*}
Let $\sigma^2(q)=  \mathbb{E}\left( \left[\max_{j=1\dots q} g_j \right]^2\right)$, we have finally:
\begin{equation}\label{eq:Expected}
\mathbb{E}(h(x)h(z))=\sigma^2(q) \scalT{x}{z} \mathbb{P}(D(x)=D(z)).
\end{equation}
$\sigma^2(q)$ is a normalization factor and it is well known that $\sigma^2(q)\sim \log(q)$, hence we are left with 
$$\mathbb{P}(D(x)=D(z)),$$
that is the probability that $x$ and $z$ are not separated by the $q$ hyperplanes, an object that is well studied in $q-$ ways graph cuts approximation algorithms. \\
\noindent The following  lemma is crucial to our proof and  is proved in \cite{Frieze95}, and  allow us to get the final expression of the expected kernel.
\begin{lemma}[\cite{Frieze95}]
For $x,z \in \mathbb{R}^d, ||x||=||z||=1$. Let $\rho=\scalT{x}{z}$, we have therefore:
\begin{equation}
\kappa_q(x,z)=\mathbb{P}\left(D(x)=D(z)\right)=\sum_{i=0}^{\infty}a_i(q)\rho^i
\end{equation}
the taylor series of $\kappa_q$ around $\rho=0$, converges for all $\rho$ in the range $\abs{\rho} \leq 1$.
The coefficients $a_i(q)$, of the expansion are all non negatives and their sum converges to $1$. 
The first $3$ coefficients are $a_0(q)=\frac{1}{q},a_1(q)=\frac{h_1^2(q)}{q-1}$, $a_2(q)=\frac{qh_2^2(q)}{(q-1)(q-2)}$.
$h_{i}(q)=\mathbb{E}\phi_i(\max_{k=j\dots q}\eta_j))$ , where $\eta_{j},j=1\dots q$ are iid standard centered gaussian, and $\phi_{i}$, the normalized Hermite polynomials. 
\label{l:Maxcut}
\end{lemma}
\noindent By lemma \ref{l:Maxcut} we have finally:
\begin{equation}
K(x,z)=\mathbb{E}(h(x)h(z))= \sigma^2(q)\scalT{x}{z}\kappa_q(x,z),
\end{equation}
where $\kappa_q(x,z)= \sum_{i=0}^{\infty}a_i(q)(\scalT{x}{z})^i$, is a non-linear kernel, values of $a_i(q)$ are given in the above lemma. 
\end{proof}

\section{Learning with Random Maxout Features}
%
In this section we state the proof of Theorem 2.
We start with a preliminary Lemma that bounds $\phi(x,W)$ uniformly on the set $\XX$, this will be crucial in our derivations.
\begin{lemma}[Bounding $\sup_{x\in \mathcal{M} }\left|\phi(x,W)\right|$]\label{lem:boundsup}
Let $\mathcal{M}=\XX \cap  \mathbb{S}^{d-1}$. Let $d_{\mathcal{M}}$ be the Assouad dimension of $\mathcal{M}$ and $\rm{diam}(\mathcal{M})$ be the \rm{diam}eter of $\mathcal{M}$. Let $\delta>0$, we have  for a numeric constant $C_1$: 
$$\sup_{x\in \mathcal{M}} \left|\phi(x,W)\right|=\sup_{x\in \mathcal{M}} \left|\max_{j=1\dots q}\scalT{w_j}{x}\right| \leq 	C_1 \sqrt{ d_{\mathcal{M}}\log \left(\frac{{\rm{diam}}(\mathcal{M})\sqrt{d}}{\delta}\right)+\log(q+1)},$$ with probability at least $1-  \delta-2 e^{-cd/4}.$

\end{lemma}
\begin{proof}
Consider an $\epsilon-$Net that covers $\XX$ with balls of radius $r$ and centers $\{x_i\}_{i=1\dots T}$.  We have by definition of the Assouad dimension, the maximum number of balls $T$ is less than  
$\left(\frac{2 \rm{diam}(\mathcal{M})}{r}\right)^{d_{\mathcal{M}}}$.
 Assume we have: $|\phi(x,W)-h(z,W)|=\scalT{w_{D(x)}}{x}-\scalT{w_{D(z)}}{z}$, meaning $\scalT{w_{D(x)}}{x}-\scalT{w_{D(z)}}{z}>0$. 
\begin{eqnarray*}
\phi(x,W)-\phi(z,W)&=&\scalT{w_{D(x)}}{x}-\scalT{w_{D(z)}}{z}\\
&=& \scalT{W_{D(x)}}{x-z}\\&-& \underbrace{\scalT{w_{D(z)}-w_{D(x)}}{z}}_{\geq 0} \\
&\leq& ||w_{D(x)}||_{2} \nor{x-z}_{2} ,
\end{eqnarray*}
where the inequality follows from the definition of $D(z)$, and the Cauchy-Schawrz inequality.
Similarly if we have $|\phi(x,W)-\phi(z,W)|=\scalT{w_{D(z)}}{z}-\scalT{w_{D(x)}}{x}$, we have:
\begin{eqnarray*}
\phi(x,W)-\phi(z,W))&\leq& ||w_{D(z)}||_{2} ||x-z||_{2} 
\end{eqnarray*}
We conclude therefore that: 
\begin{equation*}
\left|\phi(x,W)-\phi(z,W)\right|\leq \max\left(\nor{w_{D(x)}},\nor{w_{D(z)}}\right) \nor{x-z}_{2}\leq\left( \nor{w_{D(x)}}+ \nor{w_{D(z)}} \right)  \nor{x-z}_{2}
\end{equation*}
Let $L= \nor{w_{D(x)}}_2+ \nor{w_{D(z)}}_2$.

Let $t>0$, we have $\sup_{x \in \mathcal{M}}   \left|\phi(x,W)\right| < t$, if the following two events hold:
$$E_{1}=\left\{ \sup_{x_i, i=1\dots T}\left| \phi(x_i,W)\right|  < \frac{t}{2}\right\} \text{and } E_2=\left\{ L \leq \frac{t}{2r}\right\}.$$

On the first hand: 
\begin{align*}
\mathbb{P}(E^c_1)&= \mathbb{P} \left(  \sup_{x_i, i=1\dots T} \left| \phi(x_i,W)\right|  \geq  \frac{t}{2}\right)\\
&= \mathbb{P}\left(\cup_{i=1}^T \{ \left|\phi(x_i,W)\right|  \geq  \frac{t}{2} \}\right)\\
&\leq \sum_{i=1}^T \mathbb{P}\left( \left|\phi(x_i,W)\right| \geq  \frac{t}{2} \right) \\
&= T \mathbb{P}\left(\left|\max_{j=1\dots q } \scalT{w_j}{x} \right| \geq  \frac{t}{2} \right).
\end{align*}
Note that by a union bound we have: 
$$\mathbb{P}\left(\max_{j=1\dots q } \scalT{w_j}{x}  \geq  \frac{t}{2} \right)= \mathbb{P}\left(\exists j , \scalT{w_j}{x} \geq \frac{t}{2}\right) \leq q \mathbb{P}(\scalT{w}{x}\geq \frac{t}{2}) \leq q e^{-t^2/8},$$
and by independence of $w_j$ we have also:
$$\mathbb{P}\left(\max_{j=1\dots q } \scalT{w_j}{x}  \leq - \frac{t}{2} \right)= \mathbb{P}\left(\forall j , \scalT{w_j}{x} \leq- \frac{t}{2}\right) =\left( \mathbb{P}(\scalT{w}{x}\leq- \frac{t}{2})\right)^q \leq  e^{-qt^2/8}.$$
Putting together theses to bounds we have:
 $$\mathbb{P}\left(\left|\max_{j=1\dots q } \scalT{w_j}{x} \right| \geq  \frac{t}{2} \right)\leq qe^{-t^2/8}+ e^{-qt^2/8}.$$
 The covering number $T$ of $\XX$ is also bounded as follows \cite{Assouad}:
 $$T \leq \left(\frac{2\rm{diam}(\mathcal{M})}{r}\right)^{d_{\mathcal{M}}}.$$
Hence we have for $q>1$:
$$\mathbb{P}(E^c_1) \leq \left(\frac{2\rm{diam}(\mathcal{M})}{r}\right)^{d_{\mathcal{M}}} \left( qe^{-t^2/8}+ e^{-qt^2/8} \right)\leq \left(\frac{2\rm{diam}(\mathcal{M})}{r}\right)^{d_{\mathcal{M}}}(q+1)e^{-t^2/8} .$$
On the other hand,  for a universal constant $c$, and for $\eps \in (0,1)$ \cite{vershyninreview}:
$$\mathbb{P}(\nor{w}_2 \geq \sqrt{d}(1+\eps))\leq e^{-c\eps^2 d}.$$
Set $\frac{t}{2r}=\sqrt{d}(1+\eps)$, hence $\mathbb{P}(E^c_2) \leq 2 e^{-c\eps^2d}$.\\
It follows that  for $t>1$:
\begin{align*}
\mathbb{P}(\sup_{x\in \mathcal{M}} \left|\phi(x,W)\right| \geq t ) &\leq \mathbb{P}(E^1_c \cup E^2_c)\\
&\leq \mathbb{P}(E^1_c) +\mathbb{P}(E^2_c)\\
&\leq \left(\frac{4 \sqrt{d}(1+\eps)\rm{diam}(\mathcal{M})}{t}\right)^{d_{\mathcal{M}}} (q+1)e^{-t^2/8}+2 e^{-c\eps^2d}\\
&\leq \left(4\sqrt{d}(1+\eps)\rm{diam}(\mathcal{M})\right)^{d_{\mathcal{M}}}(q+1)e^{-t^2/8}+2 e^{-c\eps^2d}.
\end{align*}

Hence for $\eps=\frac{1}{2}, t>1$:
$$\sup_{x\in \mathcal{M}} \left|\phi(x,W)\right| \leq t , $$ with probability at least $1-  \left(6  { \rm{diam}}(\mathcal{M})\sqrt{d}\right)^{d_{\mathcal{M}}} (q+1)e^{-t^2/8}-2 e^{-cd/4}.$ 

Hence we have  for a numeric constant $C_1$: 
$$\sup_{x\in \mathcal{M}} \left|\phi(x,W)\right| \leq 	C_1 \sqrt{ d_{\mathcal{M}}\log \left(\frac{{\rm{diam}}(\mathcal{M})\sqrt{d}}{\delta}\right)+\log(q+1)},$$ with probability at least $1-  \delta-2 e^{-cd/4}.$ 
\end{proof}
The following Lemma shows that any function $f \in \mathcal{F}$, can be approximated by a function $\hat{f} \in \hat{\mathcal{F}}$:
\begin{lemma}\label{lem:approx}[Approximation Error.]Let $f$ be a function in $\mathcal{F}$. Then for $\delta>0$, there exists a function $\hat{f} \in \hat {\mathcal{F}}$ such that:
$$\nor{\hat{f} -f}_{\mathcal{L}^2(\XX,\rho_{\mathcal{M}})}\leq CC_1  \sqrt{\frac{d_{\mathcal{M}}\log \left(\frac{{\rm{diam}}(\mathcal{M})\sqrt{d}}{\delta}\right)+\log(q+1)}{{m}}}\left(1+\sqrt{2\log\left(\frac{1}{\delta}\right)}\right)$$
with probability at least $1-2\delta-2 e^{-cd/4}$.
\end{lemma}
\begin{proof}[Proof of Lemma \ref{lem:approx}] Let $f \in \mathcal{F}, f(x)= \int \alpha(W) \phi(x,W) dW $.$\text{ Let } {f_{\ell}}(x)=\frac{ \alpha(W^{\ell})}{p(W^{\ell})}\phi(x,W_{\ell}).$We have the following: $\mathbb{E}_{W}(f_{\ell})=f$, and $\frac{1}{m}\mathbb{E}_{W}(\sum_{\ell=1}^m f_{\ell})=f$.
Consider the Hilbert space $\mathcal{L}^2(\XX,\rho_{\mathcal{M}})$, with dot product:
$\scalT{f}{g}_{\mathcal{L}^2(\XX,\rho_{\mathcal{M}})}=\int_{\XX} f(x)g(x)d\rho_{\mathcal{M}}(x)$.\\
$$|| f_{\ell}||_{\mathcal{L}^2(\XX,\rho_{\mathcal{M}})} = \sqrt{\int_{\XX} \left( \frac{\alpha(W^{\ell})}{p(W^{\ell})}\right)^2 (\phi(x,W^{\ell}))^2 d\rho_{\mathcal{M}}(x)},$$
Let $E$ and $F$ be the event defined as follows: 
$$E=\left\{ \sup_{x\in \mathcal{M}} \left|\phi(x,W)\right|) \leq M \right\}$$
$$ F= \left\{ \nor{\frac{1}{m}\sum_{j=1}^m f_j -f}_{\mathcal{L}^2(\XX,\rho_{\mathcal{M}})}>t\right\}$$
Conditioned on E we have:
$$|| f_{\ell}||_{\mathcal{L}^2(\XX,\rho_{\mathcal{M}})} \leq CM. $$
$$ \mathbb{P}\left(F\right) = \mathbb{P}\left( F|E\right)\mathbb{P}(E)+  \mathbb{P}\left( F|E^c\right)\mathbb{P}(E^c)\leq  \mathbb{P}\left( F|E\right)+\mathbb{P}(E^c).$$
Conditioned on the event $E$, we can apply McDiarmid inequality and we have:
$$\mathbb{P}(F|E) \leq \exp\left(-\frac{mt^2}{2M^2C^2}\right)=\delta_1 $$
For $\delta>0$ set $M= 	C_1 \sqrt{ d_{\mathcal{M}}\log \left(\frac{{\rm{diam}}(\mathcal{M})\sqrt{d}}{\delta}\right)+\log(q+1)} $ applying Lemma \ref{lem:boundsup} we have:
$$\mathbb{P}(E^c) \leq 1-\delta -2e^{-c d/4}$$

We have therefore with probability $1-\delta-2e^{-c d/4}-\delta_1$:\\
\begin{equation}\label{eq:1}
\nor{\frac{1}{m}\sum_{j=1}^m f_j -f}_{\mathcal{L}^2(\XX,\rho_{\mathcal{M}})}\leq CC_1  \sqrt{\frac{d_{\mathcal{M}}\log \left(\frac{{\rm{diam}}(\mathcal{M})\sqrt{d}}{\delta}\right)+\log(q+1)}{{m}}}\left(1+\sqrt{2\log\left(\frac{1}{\delta_1}\right)}\right).
\end{equation}
\end{proof}

The following Lemma shows how the approximation of functions in $\mathcal{F}$, by functions in $\hat{\mathcal{F}}$, transfers to the expected Risk:
\begin{lemma}[Bound on the Approximation Error]\label{lem:comp}
Let $f\in \mathcal{F}$, fix $\delta>0$. There exists a function $\hat{f} \in \hat{\mathcal{F}}$, such that:
$$\mathcal{E}_{V}(\hat{f})\leq \mathcal{E}_{V}(f)+LCC_1  \sqrt{\frac{d_{\mathcal{M}}\log \left(\frac{{\rm{diam}}(\mathcal{M})\sqrt{d}}{\delta}\right)+\log(q+1)}{{m}}}\left(1+\sqrt{2\log\left(\frac{1}{\delta}\right)}\right) $$
with probability at least $1-2\delta-2 e^{-cd/4}$.
\end{lemma}
\begin{proof}[Proof of Lemma \ref{lem:comp}]$\mathcal{E}_{V}(\hat{f})-\mathcal{E}_{V}(f) \leq \int_{\XX}\left|V(y\hat{f}(x))-V(yf(x))\right|d\rho_{\mathcal{M}}(x)\leq L \int_{\XX}|\hat{f}(x)-f(x)|d\rho_{\mathcal{M}}(x)\leq L \sqrt{\int_{\XX}(\hat{f}(x)-f(x))^2d\rho_{\mathcal{M}}(x)}=L \nor{\hat{f}-f}_{\mathcal{L}^2(\XX,\rho_{\mathcal{M}})},$ where we used the Lipchitz condition and Jensen inequality. The rest of the proof follows from Lemma \ref{lem:approx}. 
\end{proof}


We are now ready to prove Theorem $2$.
\begin{proof} [Proof of Theorem $2$]
Let $\hat{f}_{N}=\argmin_{f\in \hat{\mathcal{F}}}\hat{\EE}_{V}(f)$,  $\hat{f}=\argmin_{f\in \hat{\mathcal{F}}}\mathcal{E}_{V}(f)$,
$f^*=\argmin_{f\in \mathcal{F}}\mathcal{E}_{V}(f)$.
\begin{align*}
\mathcal{E}_{V}(\hat{f}_{N})-\min_{f\in \mathcal{F}}\mathcal{E}_{V}(f)&=\underbrace{\left(\mathcal{E}_{V}(\hat{f}_{N})-\mathcal{E}_{V}(\hat{f})\right)}_{\text{Statistical Error}}+\underbrace{\left(\mathcal{E}_{V}(\hat{f})-\mathcal{E}_{V}(f^*)\right)}_{\text{Approximation Error}}
\end{align*}
\textbf{Bounding the statistical error.}
The first term is the usual estimation or statistical error than we can bound as follows:
\begin{align*}
\mathcal{E}_{V}(\hat{f}_{N})-\mathcal{E}_{V}(\hat{f})&=\left(\mathcal{E}_{V}(\hat{f}_{N})- \hat{\EE}_{V}(\hat{f}_{N})\right)+\underbrace{\left( \hat{\EE}_{V}(\hat{f}_{N})-\hat{\EE}_{V}(\hat{f})\right)}_{\leq 0,\text{by optimality of $\hat{f}_{N}$}}+\left(\hat{\EE}_{V}(\hat{f})-\EE_{V}(\hat{f})\right) \\
&\leq 2\sup_{f\in \hat{\mathcal{F}}} \left|\mathcal{E}_{V}(f)-\hat{\EE}_{V}(f)\right|.
\end{align*}
Assume that the loss $V: \mathbb{R}\to [0,1]$, when the data $(x_i,y_i)$ or the random projections $W^{\ell}$ change $\sup_{f\in \hat{\mathcal{F}}} \left|\mathcal{E}_{V}(f)-\hat{\EE}_{V}(f)\right|$  , can change by no more then $\frac{2}{N}$  then by applying McDiarmids inequality \cite{Radm} we have with a probability at least $1-\delta/2$

$$\sup_{f\in \hat{\mathcal{F}}} \left|\mathcal{E}_{V}(f)-\hat{\EE}_{V}(f)\right| \leq \mathbb{E}_{x,W}\left(\sup_{f\in \hat{\mathcal{F}}} \left|\mathcal{E}_{V}(f)-\hat{\EE}_{V}(f)\right|\right) +\sqrt{\frac{2\log(2/\delta)}{N}} .$$
Now using the classical rademachar complexity type bounds \cite{Radm}, we have:
$$\mathbb{E}_{x,W}\sup_{f\in \hat{\mathcal{F}}} \left|\mathcal{E}_{V}(f)-\hat{\EE}_{V}(f)\right| \leq2 L \mathcal{R}_{N}(\hat{\mathcal{F}})+ \frac{|V(0)|}{\sqrt{N}},$$
where $\mathcal{R}_{N}(\hat{\mathcal{F}})$ is defined as follows:
$$\mathcal{R}_{N}(\hat{\mathcal{F}})=\mathbb{E}_{x,W,\sigma}\left[\sup_{f\in\hat{\mathcal{F}}}\left|\frac{1}{N}\sum_{i=1}^N \sigma_i f(x_i)\right|\right],$$
where $\sigma_i$ are iid  Rademacher variables $\in \{-1,1\}$, such that $\mathbb{P}(\sigma_i=1)=\mathbb{P}(\sigma_i=-1)=\frac{1}{2}$.

It is sufficient to bound the Rademacher complexity of the class $\hat{\mathcal{F}}$, where the expectation is taken over the randomness of the data and the random features:
\begin{align*}
\mathcal{R}_{N}(\hat{\mathcal{F}})&=\mathbb{E}_{x,W,\sigma}\left[\sup_{f\in\hat{\mathcal{F}}}\left|\frac{1}{N}\sum_{i=1}^N \sigma_i f(x_i)\right|\right]=\mathbb{E}_{x,W,\sigma}\left[\sup_{f\in\hat{\mathcal{F}}}\left|\frac{1}{N}\sum_{i=1}^N \sigma_i\left(\sum_{\ell=1}^m \alpha_{\ell} \phi(x_i,W^{\ell})\right) \right|\right]\\
&= \mathbb{E}_{x,W,\sigma}\left[\sup_{f\in\hat{\mathcal{F}}}\left|\frac{1}{N}\sum_{\ell=1}^m \alpha_{\ell} \sum_{i=1}^N \sigma_i \phi\left(x_i,W^{\ell}\right) \right|\right]\\
&\leq \mathbb{E}_{x,W,\sigma} \frac{1}{N}\nor{\alpha}_{\infty}\sum_{\ell=1}^m \left|\sum_{i=1}^N \sigma_i\phi\left(x_i,W^{\ell}\right)\right| \text{ By Holder inequality: $\scalT{a}{b}\leq \nor{a}_{\infty}\nor{b}_{1}$}\\
&\leq \frac{C}{mN} \mathbb{E}_{x,W}\sum_{\ell=1}^m \sqrt{\mathbb{E}_{\sigma}\left(\sum_{i=1}^N \sigma_i\phi\left(x_i,W^{\ell}\right)\right)^2} \text{Jensen inequality, concavity of square root}
\end{align*} 
Note that $\mathbb{E}(\sigma_i\sigma_j)=0$, for $i \neq j$ it follows that:\\
$\mathbb{E}_{\sigma}\left(\sum_{i=1}^N \sigma_i\phi\left(x_i,W^{\ell}\right)\right)^2= \mathbb{E}_{\sigma} \sum_{i=1}^N \sum_{j=1}^N \sigma_i\sigma_j \phi\left(x_i,W^{\ell}\right)\phi\left(x_j,W^{\ell}\right)= \sum_{i=1}^N \phi^2\left(x_i,W^{\ell}\right)$.
Finally: 
\begin{align*}
\mathcal{R}_{N}(\hat{\mathcal{F}})&\leq \frac{C}{{mN}} \sum_{\ell=1}^m \mathbb{E}_{x,W}\left(\sqrt{ \sum_{i=1}^N \phi^2\left(x_i,W^{\ell}\right)}\right)\\
&= \frac{C}{N}\mathbb{E}_{x,W}\left(\sqrt{ \sum_{i=1}^N \phi^2\left(x_i,W\right)}\right)\\
&\leq \frac{C}{N}\sqrt{\mathbb{E}_{x,W}\left(\sum_{i=1}^N \phi^2(x_i,W)\right)} \text{By Jensen inequality }\\
&=  \frac{C}{N}\sqrt{N \mathbb{E}_{x,W}\phi^2(x,W)}\\
&\leq  \frac{C}{\sqrt{N}} \sqrt{\mathbb{E}_{x} \left(K(x,x)\right)}.
\end{align*}
Recall that for $x\in \mathbb{S}^{d-1}$, $K(x,x)=\sigma^2(q)\nor{x}^2 \kappa_{q}(x,x)=\sigma^2(q)$.
Hence:
$$\mathcal{R}_{m}(\hat{\mathcal{F}})\leq C \sqrt{\frac{\sigma^2(q)}{N}},$$
hence we have with probability $1-\delta/2$, on the choice of random data and random projections:
\begin{equation}
\mathcal{E}_{V}(\hat{f}_{N})-\mathcal{E}_{V}(\hat{f})\leq 4LC \sqrt{\frac{\sigma^2(q)}{N}}+ \frac{2|V(0)|}{\sqrt{N}}+ 2\sqrt{\frac{2\log(2/\delta)}{N}}.
\end{equation}
\textbf{Bounding the Approximation Error.} 
Let $\hat{f}^*$, the function defined in Lemma \ref{lem:approx}, that approximates $f^*$ in $\hat{\mathcal{F}}$.  
By Lemma \ref{lem:comp} we know that:
$$\mathcal{E}_{V}(\hat{f}^*)\leq \mathcal{E}_{V}(f^*)+LCC_1  \sqrt{\frac{d_{\mathcal{M}}\log \left(\frac{{\rm{diam}}(\mathcal{M})\sqrt{d}}{\delta}\right)+\log(q+1)}{{m}}}\left(1+\sqrt{2\log\left(\frac{1}{\delta}\right)}\right) $$
with probability $1-2\delta-2 e^{-cd/4}$, on the choice of the random projections.
By optimality of $\hat{f} \in \tilde{\mathcal{F}}$, we have  with at least the same probability $1-2\delta-2 e^{-cd/4}$ $$\EE_{V}(\hat{f})\leq \EE_{V}(\hat{f}^*)\leq \mathcal{E}_{V}(f^*) +LCC_1  \sqrt{\frac{d_{\mathcal{M}}\log \left(\frac{{\rm{diam}}(\mathcal{M})\sqrt{d}}{\delta}\right)+\log(q+1)}{{m}}}\left(1+\sqrt{2\log\left(\frac{1}{\delta}\right)}\right)$$
Hence by a union bound with probability $1-3\delta -2 e^{-cd/4}$, on the training set and the random projections:
\begin{align*}
\mathcal{E}_{V}(\hat{f}_{N})- \min_{f\in \mathcal{F}}\mathcal{E}_{V}(\hat{f}) &\leq  4LC \sqrt{\frac{\sigma^2(q)}{N}}+ \frac{2|V(0)|}{\sqrt{N}}+ 2\sqrt{\frac{2\log(1/\delta)}{N}}\\
&+ LCC_1  \sqrt{\frac{d_{\mathcal{M}}\log \left(\frac{{\rm{diam}}(\mathcal{M})\sqrt{d}}{\delta}\right)+\log(q+1)}{{m}}}\left(1+\sqrt{2\log\left(\frac{1}{\delta}\right)}\right).
\end{align*}
\end{proof}

\bibliographystyle{ieeetr}


\end{document}